\documentclass[runningheads]{llncs}
\usepackage[T1]{fontenc}
\usepackage{graphicx}
\usepackage{siunitx}
\usepackage{amsmath}
\usepackage{amssymb}
\usepackage{dyatex}
\usepackage{ntdmath}
\usepackage{sequat}
\usepackage{petrinet}
\usepackage[
pdfusetitle,
colorlinks,
bookmarksopen,
bookmarksnumbered,
citecolor=gray,
urlcolor=gray,
linkcolor=gray]{hyperref}
\usepackage{tcolorbox}
\usepackage{enumitem}
\usepackage{cite}
\tcbset{colback=gray!5!white,colframe=gray!75!black,fonttitle=\bfseries}

\usepackage{subcaption}
\usepackage{tabularx}
\usepackage{booktabs}
\usepackage{multirow}
\usepackage{multicol}
\usepackage{wrapfig}

\usepackage{graphicx}
\usepackage{caption}

\newcommand{\varTsei}{\gamma} 
\newcommand{\varLiteral}{\ell}
\newcommand{\varExp}{\phi}

\newcommand{\varPnTrans}{\tau}
\newcommand{\varPnPlace}{p}
\newcommand{\marking}{\VEC{p}}
\newcommand{\transitions}{\boldsymbol{\tau}}
\newcommand{\weightasym}{C}
\newcommand{\weighta}{\MAT{\weightasym}}
\newcommand{\pval}[0]{\DataSty{partial-eval}}

\newcommand{\psat}[0]{\DataSty{partial-sat}}

\newcommand{\arcweight}{\emph{w}}

\newcommand{\infexplanation}{\ensuremath{\xi}}
\newcommand{\infexplanationspace}{\ensuremath{\Xi}}

\newcommand{\relaxedtransitions}{\tilde{\boldsymbol{\tau}}}

\newcommand{\slackp}{\emph{s}^+}
\newcommand{\slackn}{\emph{s}^-}
\newcommand{\slackpv}{\bold{s}^+}
\newcommand{\slacknv}{\bold{s}^-}

\newcommand{\updates}{\mathcal{U}}

\renewcommand{\actspace}{\mathcal{A}}

\newcommand{\planprob}[0]{\ensuremath{\mathcal{P}}}
\newcommand{\planupdate}[0]{\ensuremath{\mathcal{U}}}

\newcommand{\planstart}[0]{\ensuremath{I}}
\newcommand{\plangoal}[0]{\ensuremath{G}}
\newcommand{\plangoalclause}[0]{\ensuremath{g}}
\newcommand{\planround}[2]{{#1}^{#2}}
\newcommand{\planroundend}[0]{n}
\newcommand{\probindex}[0]{i}
\newcommand{\stepvar}[0]{k}
\newcommand{\stephorizon}[0]{h}

\newcommand{\pbindingsstep}[1]{\discstep{\mathcal{B}}{#1}}

\newcommand{\pddlActions}{\ensuremath{A}}
\newcommand{\pddlaction}{\ensuremath{a}}
\newcommand{\pddlVars}{\ensuremath{V}}
\newcommand{\pddlvariable}{\ensuremath{v}}

\begin{document}
\title{Petri Net Relaxation for Infeasibility Explanation and Sequential Task Planning}
%
%
\author{Nhat Le\inst{1}\orcidID{0009-0007-1013-3711} %
\and John G. Rogers\inst{2}\orcidID{0000-0002-6074-0823} %
\and Claire N. Bonial\inst{2}\orcidID{ 0000-0002-3154-2852} %
\and Neil T. Dantam\inst{1}\orcidID{0000-0002-0907-2241}}
\authorrunning{Le, Rogers, Bonial, and Dantam}
%
\institute{Colorado School of Mines, Golden CO 80401, USA\\
  \email{\{nguyencongnhat\_le,ndantam\}@mines.edu}
  \and DEVCOM Army Research Laboratory, Adelphi, MD, USA\\
  \email{\{john.g.rogers59.civ,claire.n.bonial.civ\}@army.mil}
}
\maketitle              
\begin{abstract}
  Plans often change due to changes in the situation or our
  understanding of the situation. Sometimes, a feasible plan may not
  even exist, and identifying such infeasibilities is useful to
  determine when requirements need adjustment.
  Common planning approaches focus on efficient one-shot planning in
  feasible cases rather than updating domains or detecting
  infeasibility.
  We propose a Petri net reachability relaxation to enable robust
  invariant synthesis, efficient goal-unreachability detection, and
  helpful infeasibility explanations.
  We further leverage incremental constraint solvers to support goal
  and constraint updates. Empirically, compared to baselines, our
  system produces a comparable number of invariants, detects up to
  $2\times$ more infeasibilities, performs competitively in one-shot
  planning, and outperforms in sequential plan updates in the tested
  domains.

\keywords{Task Planning \and Algorithmic Completeness and Complexity}
\end{abstract}

\section{Introduction}

The classic line ``Kein Operationsplan reicht mit einiger Sicherheit
\"{u}ber das erste Zusammentreffen mit der feindlichen Hauptmacht
hinaus,'' roughly translates to ``no plan survives contact with the
enemy''~\cite{moltke1900military}.
Across applications, plans must often change, because the situation
has changed, or because a user did not initially understand or correctly
specify the situation.
Further, it is useful to identify cases when no plan exists, and why.
Planning research has largely considered a given specification or
model, where our only want is to find a plan, ideally
fast~\cite{vallati20152014,cenamor2019insights,ipc2023}.
Dynamic navigation~\cite{koenig2002d, stentz1995focussed,
  ziebart2008fast} and motion~\cite{mohanan2018survey,
  chiang2019safety} planning are well-studied, yet adapting to
analogous changes in task planning is less
explored~\cite{soutchanski2025planning}.
We address this challenge of sequentially planning or explaining
infeasibility and updating task planning problems.

\dyathesis{We relax Petri net reachability to identify invariants and
  explain infeasibilities in combinatorial and numeric task planning,
  coupled with incremental constraint solving for sequential planning
  to show improved running times, improved infeasibility detection,
  and new infeasibility explanation}.
Transforming planning problems to Petri nets (PNs) reveals useful
problem structure.
Relaxing PN reachability to a linear program (LP) effectively
identifies invariants and infeasible problems, and conflicting LP
contraints offer useful explanations for planning infeasibilities.
Incremental constraint solving efficiently incorporates certain forms
of problem updates.
Compared to evaluated baseline planners, we show competitive running
times for some single-shot numeric domains, improved running times for
sequential planning on tested domains, improved infeasibility
detection on tested domains, and new capabilities to explain
infeasibility.

\section{Related Work}

Task planning is well-established, largely evolving from the
pioneering work on STRIPS \cite{fikes1972strips}.  Efficient task
planning approaches include heuristic search
\cite{helmert2006fast,hoffmann2001ff,scala2016interval} and
constraint-based methods
\cite{kautz1999unifying,rintanen2012engineering,rintanen2014madagascar,bofill2016rantanplan,hsu2007constraint}.
Planning domains are often specified using the Planning Domain
Definition Language (PDDL)
\cite{edelkamp2004pddl2,mcdermott1998pddl,haslum2019introduction}.

Prior approaches have also posed planning problems as integer programs
(IPs)~\cite{van2005reviving,spevak2022ivenness}.
\cite{vossen2000applying,vossen1999use} propose an IP formulation for
with an effective LP relaxation for plan-length estimation.
We develop a different LP formulation for invariant and infeasibility
detection.

Identifying infeasible planning problems has received less attention
than planning for feasible cases.
Some approaches focus on human-understandable
explanations~\cite{sreedharan2019can, valmeekam2022radar}.
\cite{eriksson2017unsolvability} produces unsolvability certificates,
focusing on model verification.
We develop a polynomial-time infeasibility test that is empirically
robust and further identifies conflicting constraints.

Prior work has addressed planning in dynamic and incomplete spaces.
Building on heuristic search~\cite{hart1968formal}, incremental
approaches~\cite{stentz1994d, koenig2002d, likhachev2008anytime} find
optimal solutions in dynamic graphs under a fixed goal
assumption, with extensions for changing
targets~\cite{koenig2007speeding,sun2010moving}.
We take a constraint-based approach to leverage efficient solution
techniques to dynamically change goals and
constraints~\cite{barbosa2022cvc5,de2008z3, gurobi}.

\section{Problem Definition} \label{sec:def}

We address \emph{sequential task planning} which involves finding
sequences of actions from a start to a goal, or determining no such
sequence of actions exists, for a sequence of related planning
problems.
Often, planning is viewed as a one-shot process, yet we may need to
iteratively revise planning problems---e.g., changing goals or
constraints---due to changes in the environment or to effectively
capture user intent.
We define one-shot and then sequential planning.

A planning problem
$\mathcal{P} =
\lmtuple{\mathcal{X}}{\actspace}{f}{\mathcal{I}}{\mathcal{G}}
{\mathcal{O}}$ consists of state space $\mathcal{X}$, actions
$\actspace$, transition function
$f: \mathcal{X\times A \rightarrow X}$, initial condition
$\mathcal{I} \in \mathcal{X}$, goals
$\mathcal{G} \subset \mathcal{X}$, and an objective
$\mathcal{O}$~\cite{lavalle2006planning}.
A feasible plan is a sequence of actions from the initial state
$\mathcal{I}$ to a goal state in $\mathcal{G}$ following applications of
transition function $f$, and an optimal plan is a feasible plan that
also maximizes (minimizes) objective $\mathcal{O}$.

We consider sequential planning as a process of solving a sequence of
related planning problems.  We begin with a given initial planning
problem $\planround{\planprob}{0}$.  Then, we are given update
sequence
$\planround{\planupdate}{1},\ldots,\planround{\planupdate}{\planroundend}$.
The planning problem at the $\probindex^{\text{th}}$ round is the
recursive composition of the $\probindex^{\text{th}}$ update
$\planround{\planupdate}{\probindex}$ and the previous problem
$\planround{\planprob}{\probindex-1}$,
\begin{equation} \label{eq:problem-update}
  \planround{\planprob}{\probindex}
  =
  \begin{cases}
    \planround{\planprob}{0},
    & \mathbf{if}\ \probindex=0 \\
    \FuncSty{compose}
    \lmtuple{\planround{\planprob}{\probindex-1}}{\planround{\planupdate}{\probindex}},
    & \mathbf{if}\ \probindex \geq 1
  \end{cases}
    \;.
\end{equation}
See \autoref{sec:incremental} for the specific updates and
compositions in our approach.

At each round $\probindex$, we must solve
$\planround{\planprob}{\probindex}$ either by (1) finding a feasible
(optimal) plan or (2) determining that no plan exists and
explaining,
$\planround{\infexplanation}{\probindex} \in \infexplanationspace$,
the infeasibility,
\begin{equation}\label{eq:plan-output}
  \FuncSty{plan}\left(\planround{\planprob}{\probindex}\right)
  \in
    \actspace^*
  \cup
  \left(
  \lmset{\DataSty{UNSAT}}
  \times \infexplanationspace
  \right)
  \; .
\end{equation}
See \autoref{ssec:infeasibility-expl} for the infeasibility
explanations we produce.

Sequential planning terminates when
$\FuncSty{plan}\left(\planround{\planprob}{\planroundend}\right)$
produces a desired result.

\section{Background}
\paragraph{Planning Domain Definition Language}
\label{sec:pddl}
The Planning Domain Definition Language (PDDL) is a de-facto standard
notion for describing planning problems~\cite{mcdermott1998pddl}.
PDDL descriptions use first-order logic to specify the planning
problem, though we consider the grounded or propositionalized form
$\lmtuple{\pddlVars}{I}{A}{G}$, where $\pddlVars$ is a set of state
variables that may be Boolean- or real-valued, $I$ is an initial
state, $\pddlActions$ is a set of actions, and $G$ is a set of goal
states. State $s$ is a binding to each state variable.  Actions
$\pddlaction \in \pddlActions$ include preconditions \emph{pre} and
effects \emph{eff}.

\paragraph{Constraint Solvers}

We consider two types of constraint formulations and solvers:
Satisfiability Modulo Theories (SMT) and Mixed Integer Programming
(MIP).
SMT extends Boolean satisfiability with rules (theories) for domains
such as enumerated types and linear
arithmetic~\cite{BarFT-RR-25,de2011satisfiability}. Compared to
traditional SAT solvers, SMT solvers provide a higher level interface
with useful features for expressing constraints
\cite{wang2016task,cashmore2016compilation,dantam2018tmp}.
MIP extends numerical programming (constrained optimization) to
include not only real-valued but also integer or Boolean decision
variables.
Though seeming to come from opposite directions, SMT and MIP offer
some convergent capabilities; however, algorithmic differences in
solvers sometimes yield substantially different performance.
SMT solvers~\cite{de2008z3,barbosa2022cvc5} are often based on a
backtracking search core~\cite{marques1999grasp,moskewicz2001chaff},
while MIP solvers~\cite{huangfu2018parallelizing,gurobi} are often
organized around a branch-and-cut framework~\cite{padberg1991branch}.
Performance differences suggest advantages to supporting multiple
solver types~\cite{rungta2022billion,combrink2025comparative}.

A useful feature of many SMT solvers is \emph{incremental} constraint
solving, enabling certain additions and deletions of constraints at
run-time to produce alternative solutions.
There are two mechanisms for incremental solving: an assertion
stack~\cite{ganzinger2004dpll} and per-check
assumptions~\cite{een2003extensible}.
An incremental SMT solver maintains constraints using a stack of
scopes, where each scope is a container for a set of constraints.  New
constraints are added to the scope on top of the stack, and solvers
offer an interface to push and pop scopes (removing constraints in any
popped scope).
Per-check assumptions are constraints that hold only for a specific
satisfiability check.
Solvers may retain or discard learned lemmas (e.g., additional
conflict clauses) differently after a check with assumptions or
popping a scope, resulting in performance differences between these
two mechanisms.

\paragraph{Constraint-based Planning}
\label{sec:bg:cstrplan}

Constraint-based planners encode the planning domain as a logical or
numeric formula and use a constraint solver, often a Boolean
satisfiability solver, to find a satisfying assignment corresponding
to a plan~\cite{lavalle2006planning,kautz1992planning}. Commonly, the
decision variables represent the state and actions for a fixed number
of steps $\stephorizon$. The formula asserts that the start state
holds at step $0$, goal condition holds at step $\stephorizon$, and
valid transitions occur between each step $\stepvar$ and $\stepvar+1$.
Planners increase step count $\stephorizon$ (creating additional
decision variables and formula clauses) until finding a satisfying
assignment encoding a plan.

A crucial part of efficient constraint-based planners
(e.g.,~\cite{kautz1999unifying,rintanen2012engineering}) is
identifying \emph{invariants}~\cite{alcazar2015reminder}, which are
facts or constraints that must hold for all reachable states at one or
more steps.  One common invariant form is \emph{mutual exclusion}
(mutex)---two Boolean state variables that cannot simultaneously be
true.
Invariants narrow the search space by eliminating decision variables
or enabling solvers to more rapidly identify bindings, resulting in
empirical runtime improvement~\cite{gerevini1998inferring}.

\paragraph{Petri Nets}

Petri nets (PNs) are a representation for transition
systems that offer useful features for
concurrency and shared resources~\cite{murata1989petri}.
A PN is a weighted, directed bipartite graph comprised of
\emph{places} and \emph{transitions}. Classically, a PN place contains
a non-negative integer-valued number of \emph{tokens}, though we
generalize to Boolean, integer, and real places.
A transition may \emph{fire} to alter the token count of adjacent
places by edge weights. A transition may only fire if the
\emph{precondition places} have enough tokens, known as the transition
cost. A \emph{marking} is the number of tokens in each place, i.e., a
PN state.
Finally, we also considers a set of goal markings for PNs, defined by
a range of tokens on each place.

PNs offer a close relationship to task planning domains beyond typical
PN applications to concurrent systems~\cite{hickmott2007planning} and
in particular
capture the implicit concurrency that prior work has found to be
critical for effective
performance~\cite{kautz1999unifying,rintanen2012engineering}. Further,
the PNs inform the linear relaxation of \autoref{ssec:relaxed}.

\section{Method} \label{sec:method}
\subsection{Petri Nets for Planning Domains} \label{ssec:pddl-pn}

\newcommand{\pddlpnalign}{\node[] at (0em,-1.75em) {$\ $};}

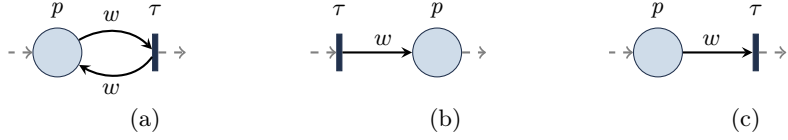
\begin{figure}
  \centering
  \begin{subfigure}[b]{.3\textwidth}
    \begin{tikzpicture}[ ]
      \node[coordinate] (box1:s) {};
      \node[pn:place,left of=box1:s] (box1 a) {};
      \node[pn:translr] at ($(box1 a)+(4em,0)$) (box1 t) {};
      \draw[pn:arc] (box1 a) edge[out=30,in=135] (box1 t);
      \draw[pn:arc] (box1 t) edge[out=-125,in=-30] (box1 a);
      \draw[pn:arcabs,<-] (box1 a) -- ++(-2em,0);
      \draw[pn:arcabs,] (box1 t) -- ++(1.25em,0);

      \node(p) at ($(box1 a)+(0,1.75em)$) {\emph{p}};
      \node(t) at ($(box1 t)+(0,1.75em)$) {\emph{$\tau$}};
      \node(w1) at ($(box1 t)+(-1.80em,1.5em)$) {\emph{w}};
      \node(w2) at ($(box1 t)+(-1.80em,-1.5em)$) {$w$};

      \pddlpnalign{}
    \end{tikzpicture}
    \vspace{-5pt}
    \caption{}
    \label{fig:pn:preeff:pre}
  \end{subfigure}
  \hfil
  \begin{subfigure}[b]{.3\textwidth}
    \begin{tikzpicture}[ ]
      \node[coordinate] (box3:s) {};
      \node[pn:translr,left of=box3:s] (box3 t) {};
      \node[pn:place] at ($(box3 t)+ (4em,0)$) (box3 a) {};
      \draw[pn:arc] (box3 t) -- (box3 a);
      \draw[pn:arcabs,<-] (box3 t) -- ++(-1.25em,0);
      \draw[pn:arcabs,] (box3 a) -- ++(2em,0);
      \node(p) at ($(box3 a)+(0,1.75em)$) {\emph{p}};
      \node(t) at ($(box3 t)+(0,1.75em)$) {\emph{$\tau$}};
      \node(w1) at ($(box3 t)+(1.80em,0.5em)$) {$w$};

      \pddlpnalign{}
    \end{tikzpicture}
    \vspace{-5pt}
    \caption{}
    \label{fig:pn:preeff:eff}
  \end{subfigure}
  \hfil
  \begin{subfigure}[b]{.3\textwidth}
    \begin{tikzpicture}[ ]
      \node[coordinate] (box2:s) {};
      \node[pn:place,left of=box2:s] (box2 a) {};
      \node[pn:translr] at ($(box2 a)+(4em,0)$) (box2 t) {};
      \draw[pn:arc] (box2 a) -- (box2 t);
      \draw[pn:arcabs,<-] (box2 a) -- ++(-2em,0);
      \draw[pn:arcabs,] (box2 t) -- ++(1.25em,0);
      \node(p) at ($(box2 a)+(0,1.75em)$) {\emph{p}};
      \node(t) at ($(box2 t)+(0,1.75em)$) {\emph{$\tau$}};
      \node(w1) at ($(box2 t)+(-1.80em,0.5em)$) {$w$};

      \pddlpnalign{}
    \end{tikzpicture}
    \vspace{-5pt}
    \caption{}
    \label{fig:pn:preeff:preeff}
  \end{subfigure}
  \vspace{-5pt}
  \caption{Petri net structure for PDDL preconditions and effects.
  (a) variable in precondition (not in effect).  (b)
  variable in effect (not in precondition).  (c) variable in
  precondition and effect.}
  \label{fig:pn:proposition}
\vspace{-10pt}
\end{figure}

We construct a PN from a grounded (propositionalized) PDDL domain (see
\autoref{sec:pddl}).  PN places correspond to PDDL state variables
$\pddlVars$, PN transitions correspond to PDDL actions $\pddlActions$,
and arcs encode preconditions and effects.

Actions involving Boolean state variables present three cases (see
\autoref{fig:pn:proposition}).
When the variable occurs only in the precondition, we create incoming
and outgoing arcs weighted $1$ for positive or $-1$ negative
preconditions (see \autoref{fig:pn:preeff:pre}).
When the variable occurs only in the effect, we create an arc from the
transition to the place weighted $1$ for positive or $-1$ negative
effects (see \autoref{fig:pn:preeff:eff}).
When the variable occurs in both the precondition and effect, we
create an arc from the place to the transition weighted $1$
to flip the variable from true to false or $-1$ to flip the variable
from false to true (see \autoref{fig:pn:preeff:preeff}).

Actions involving numeric state variables require similar arc
construction and further inference of bounds (box constraints) that
are implicitly specified in PDDL.
When the numeric variable $\pddlvariable{}$ appears in the effect with
a constant increase or decrease, we create an arc from the transition
to the place weighted by the amount of increase or decrease
(\autoref{fig:pn:preeff:eff}).
Next, we infer the bounds.
Numeric variables that never appear in decrease (increase) effects
have a lower bound (upper bound) equal to their starting value.
When every action $i$ that decreases $\pddlvariable{}$ by $x_i$ has
precondition $v \geq y_i$, we infer a lower bound of $\pddlvariable{}$
as the minimum of $y_i-x_i$.
When every action $i$ that increases $\pddlvariable{}$ by $x_i$ has
precondition $v \leq y_i$, we infer an upper bound of
$\pddlvariable{}$ as the maximum of $y_i+x_i$.

Initial state $\planstart{}$ corresponds to the PN's initial marking,
where token counts at places equal the values of corresponding state
variables in \planstart{}. Goal \plangoal{} corresponds to a goal
marking. A plan is a sequence of transitions that propagates the PN
from the initial marking to a goal marking.

This transformation defines the PN structure used throughout the paper
and underlies empirical results in \autoref{sec:result}.

\subsection{Petri Net Constraints} \label{ssec:pn-constr}

Next, we construct symbolic constraints for the PN dynamics.
While these constraints are similar to existing constraint-based
planning methods (see \autoref{sec:bg:cstrplan}), the PN formulation
supports the relaxation in \autoref{ssec:relaxed}, offering further
information about reachability and invariants.
Following typical constraint-based methods, the decision variables
represent place token counts (planning state variables) and
transitions that fire (planning actions) over steps 0 to
$\stephorizon$, and the constraint formula encodes the PN's transition
function along with start and goal conditions.

Each numeric or integer place yields a linear constraint for
token flow,
\begin{equation}
  \discstep{\varPnPlace}{\stepvar+1} = \discstep{\varPnPlace}{\stepvar} +
  \sum_{i \in \text{in}} \arcweight_i\discstep{\varPnTrans_i}{\stepvar}
  - \sum_{j \in \text{out}} \arcweight_j\discstep{\varPnTrans_j}{\stepvar}
  \; ,
  \label{eq:pn:numflow}
\end{equation}
\noindent   where
$\discstep{\varPnPlace}{\stepvar},\discstep{\varPnPlace}{\stepvar+1}$ are the token
counts of the place at steps $\stepvar,\stepvar+1$, $\varPnTrans_i$
are the incoming transitions with arc weights $\arcweight_i$, and
$\varPnTrans_j$
are the outgoing transitions with arc weights $\arcweight_j$.
Combining the constraints \eqref{eq:pn:numflow} for each place
produces a system of linear equations called the \emph{marking
  equation}~\cite{murata1989petri},
\begin{equation} \label{eq:pn:dynamics}
  \discstep{\marking}{\stepvar+1} =
  \discstep{\marking}{\stepvar} + \weighta\discstep{\transitions}{\stepvar}
  \; ,
\end{equation}
\noindent where
$\discstep{\marking}{\stepvar}, \discstep{\marking}{\stepvar+1}$ are
vectors of token counts (markings) at steps $\stepvar,\stepvar+1$,
$\discstep{\transitions}{\stepvar}$ is a vector indicating
transitions that fire, and $\weighta$ is the \emph{incidence matrix}
in which each entry $\weightasym_{ij}$ is the change in token count at
place $\varPnPlace_i$ when transition $\varPnTrans_j$ fires.

Constraints for Boolean places take a different form because we
consider as valid effects rebinding a true Boolean place to true, or
false to false.  For arc weights of 1, constraints for a Boolean place
take the following form,
\begin{equation}
  \begin{split}
  \forall i,\quad\;
  \discstep{\varPnTrans_{\mathrm{in},i}}{k} & \implies \discstep{\varPnPlace}{k+1} \\
  \forall j,\ \
  \discstep{\varPnTrans_{\mathrm{out},j}}{k} & \implies \discstep{\varPnPlace}{k} \land
                                        \lnot\discstep{\varPnPlace}{k+1} \\
  \forall \varLiteral,
  \discstep{\varPnTrans_{\mathrm{inout},\ell}}{k} & \implies \discstep{\varPnPlace}{k} \land
                                             \discstep{\varPnPlace}{k+1}
   \; ,
   \end{split}
   \label{eq:pn:boolflow}
\end{equation}
\noindent where
$\varPnTrans_{\mathrm{in}},\varPnTrans_{\mathrm{out}},\varPnTrans_{\mathrm{inout}}$
are the adjacent transitions with incoming, outgoing, and both
incoming and outgoing arcs to $p$.
Negative arc weight reverse the polarities of $\varPnPlace$ in the
implication's conclusion.

We exclude conflicting transitions from simultaneously
firing.  Conflicts may occur if one transition modifies a place in a
way that could negate a precondition of another transition.  This
mutex constraint is,
\begin{equation}
  \varPnTrans_{\mathrm{group},1}+\ldots+\varPnTrans_{\mathrm{group},{m}}
  \leq 1
  \; .
  \label{eq:mutex}
\end{equation}
We note that \eqref{eq:mutex} is different from typical mutex
constraints of pairwise disjunctions
$\lnot \varPnTrans_i \lor \lnot \varPnTrans_j $, e.g., used by
\cite{rintanen2014madagascar,kautz1998blackbox}. Pairwise disjunctions
produce quadratically-sized constraints, whereas the summation
\eqref{eq:mutex} is linear in the number of variables.  While classic
SAT solvers often required the quadratically-sized disjunctions, some
modern SMT solvers support pseudo-Boolean constraints~\cite{de2008z3}
to directly, compactly, and efficiently represent \eqref{eq:mutex}.

We also add bounds (box constraints) when known for integer and real
places,
\begin{equation}
    \min(p) \leq p \leq \max(p)
    \; .
\end{equation}

Finally, we assert the start state at step 0 and the goal at final
step $\stephorizon$.

The constraints described in this section are adequate to correctly
specify the transitions of a PN.  However, the performance of
current solvers~\cite{de2008z3,gurobi} on these constraints scales poorly and is not
practical beyond small problems.  We address this performance concern
and further offer new capabilities for reachability identification
through the linear relaxation in the next section.

\subsection{Relaxed Petri Net Reachability}%
\label{ssec:relaxed}

We describe a linear relaxation of PN reachability that offers
an effective analysis tool for planning problems.
PN reachability is well-studied and computationally
hard~\cite{blondin2020abcs,Cabasino2013}---an unsurprising observation
at this point given the reduction of task planning (also
computationally hard~\cite{bylander1994computational}) to PN
reachability.
We extend linear relaxations of PN reachability to address Boolean
state variables and show the effective identification of planning
invariants and infeasibilities.

The idea of the relaxation is to sum flows over all steps, considering
only initial marking, final marking, and how many times each
transition fired.  Intuitively, this relaxation captures a
conservation requirement: tokens are not created \emph{ex nihilo}, but
rather moved between places as transitions fire.  The relaxation does
not, however, fully capture order requirements and thus may
optimistically (and incorrectly) determine a certain state is
reachable when sums of flows could indicate so, but when no valid
ordering of the indicated transitions exists, i.e., the relaxed
solution requires some places to be negative or out-of-bounds at
intermediate steps.  We now derive the relaxed constraint equations.

Consider the marking equation \eqref{eq:pn:dynamics} applied from
steps 0 to $\stephorizon$,
\begin{equation}
  \discstep{\marking}{\stephorizon} = \discstep{\marking}{0} +
  \weighta\discstep{\transitions}{0}
  + \ldots + \weighta\discstep{\transitions}{\stephorizon-1}
  =
   \discstep{\marking}{0} + \weighta\sum_{\stepvar = 0}^{\stephorizon - 1}\discstep{\transitions}{\stepvar}
\end{equation}
A change of variables
$\relaxedtransitions = \sum_{\stepvar = 0}^{\stephorizon -
  1}\discstep{\transitions}{\stepvar}$ yields the system of equations,
\begin{equation}\label{eq:pn:relaxed:dynamics:numeric}
  \discstep{\marking}{\stephorizon} = \discstep{\marking}{0} + \weighta\relaxedtransitions
  \; .
\end{equation}

Our interpretation of Boolean places requires an additional
consideration. Under \eqref{eq:pn:boolflow}, transitions may fire that
bind a Boolean place to true (false), even if that place is already
true (false). If we would accumulate those firings as in equation
\eqref{eq:pn:relaxed:dynamics:numeric} and some other transition uses
that place in a negative (positive) precondition, then the relaxation
could miss a reachable state.
Effectively, rebindings turn \eqref{eq:pn:relaxed:dynamics:numeric}
into an inequality, which we address with non-negative slack
variables.
For each place $\varPnPlace_i$, we create a non-negative slack
variable $\slackp_i$ that increments the relaxed count and a
non-negative slack variable $\slackn_i$ that decrements the relaxed
count.
If $\varPnPlace_i$ is Boolean and no transition can rebind
$\varPnPlace_i$ to false, $\slackp_i =  0$.
If $\varPnPlace_i$ is Boolean and no transition can rebind
$\varPnPlace_i$ to true, $\slackn_i =  0$.
If $\varPnPlace_i$ is real or integer, $\slackp_i = \slackn_i = 0$.
Together, the slack variables form vectors
$\slackpv \in \mathbb{R}_+^m$ and $\slacknv \in \mathbb{R}_+^m$,
producing, 
\begin{equation} \label{eq:pn:relaxed:dynamics}
    \discstep{\marking}{h} = \discstep{\marking}{0} +
    \weighta\relaxedtransitions + \slackpv - \slacknv
    \; .
\end{equation}
Relaxing integers and Booleans to reals in
\eqref{eq:pn:relaxed:dynamics} produces a linear system.

\begin{proposition} \label{prop:inf} A marking
  $\discstep{\marking}{\stephorizon}$ resulting in infeasible
  constraints for the relaxed system dynamics
  \eqref{eq:pn:relaxed:dynamics} implies that
  $\discstep{\marking}{\stephorizon}$ is unreachable at any step under
  the original system dynamics \eqref{eq:pn:dynamics} and
  \eqref{eq:pn:boolflow}.
\end{proposition}

\begin{proof}
  We prove by contradiction.
  Assume $\discstep{\marking}{\stephorizon}$ makes relaxation
  \eqref{eq:pn:relaxed:dynamics} infeasible but is reachable under the
  original dynamics \eqref{eq:pn:dynamics} and \eqref{eq:pn:boolflow}.
  Reachability under \eqref{eq:pn:dynamics} and \eqref{eq:pn:boolflow}
  implies that there exists a sequence of transition firings to
  progressively apply \eqref{eq:pn:dynamics} and
  \eqref{eq:pn:boolflow} to produce
  $\discstep{\marking}{\stephorizon}$.  Summing those transition
  firings and selecting values for $\slackpv, \slacknv$ satisfies
  \eqref{eq:pn:relaxed:dynamics}.  So \eqref{eq:pn:relaxed:dynamics}
  must be feasible.
  Contradiction.

\end{proof}

\begin{proposition} \label{prop:polynomial} The time complexity to
  check feasibility of 
  \eqref{eq:pn:relaxed:dynamics} is polynomial in the number of places
  and transitions in the PN.
\end{proposition}
\begin{proof}
  Relaxation \eqref{eq:pn:relaxed:dynamics} consists of linear
  equations, non-negativity constraints on
  $\relaxedtransitions, \slackpv, \slacknv$, and bindings or box
  constraints on $\discstep{\marking}{\stephorizon}$, forming a linear
  program with decision variables
  ($\relaxedtransitions, \slackpv,
  \slacknv,\discstep{\marking}{\stephorizon}$) and constraints
  ($\weighta$) polynomial in the number of places and
  transitions. Linear programming is polynomial
  time~\cite{karmarkar1984new}.
\end{proof}

We next apply \eqref{eq:pn:relaxed:dynamics} to
invariants in \autoref{ssec:inv} and infeasibilities in
\autoref{ssec:infeasibility-expl}.

\subsection{Invariants from Linear Relaxations} \label{ssec:inv}


We apply the PN reachability relaxation \eqref{eq:pn:relaxed:dynamics}
to find invariants for planning domains.
Invariants are important for efficiency of
constraint-based planning~\cite{kautz1999unifying,rintanen2012engineering}.
We focus particularly on mutex invariants.

Two places are mutex when they cannot both be simultaneously true.
For places $\varPnPlace_u$ and $\varPnPlace_v$ and
\eqref{eq:pn:relaxed:dynamics}, mutex exists when,
\begin{equation}
  \left(
    \discstep{\varPnPlace_u}{\stephorizon}
    \land
    \discstep{\varPnPlace_v}{\stephorizon}
  \right)
  \models
  \left(
    \discstep{\marking}{h} \neq \discstep{\marking}{0} +
    \weighta\relaxedtransitions + \slackpv - \slacknv
    \right)
    \label{eq:inv:0}
    \; .
\end{equation}
Checking \eqref{eq:inv:0} simplifies to a linear program containing
only two rows of \eqref{eq:pn:relaxed:dynamics},

\begin{equation}
\begin{bmatrix}
  \discstep{\varPnPlace_u}{\stephorizon}\\
  \discstep{\varPnPlace_v}{\stephorizon}
\end{bmatrix}
=
\begin{bmatrix}
  \discstep{\varPnPlace_u}{0}\\
  \discstep{\varPnPlace_v}{0}
\end{bmatrix}
+
 \begin{bmatrix}
 \weighta_{u,:}\\
 \weighta_{v,:}
 \end{bmatrix}
 \relaxedtransitions
 +
\begin{bmatrix}
\slackp_u\\
\slackp_v
\end{bmatrix}
-
\begin{bmatrix}
\slackn_u\\
\slackn_v
\end{bmatrix}
\;.
\label{eq:inv:inv}
\end{equation}
By \autoref{prop:inf}, if \eqref{eq:inv:inv} is infeasible, there is
no reachable state where
both $\varPnPlace_u$ and $\varPnPlace_v$ are true, so
$\varPnPlace_u$ and $\varPnPlace_v$ must be mutex.

We note that \eqref{eq:inv:inv} offers two possible sources of
parallelism.  First, individual invariant checks may benefit from SIMD
parallelism on CPUs or GPUs~\cite{applegate2021practical}.  Second,
checking invariants across different pairs of places is embarrassingly
parallel, e.g. different checks may run in different threads using all
available cores.

Next, we use the collected pairwise mutexes
$(\lnot \varPnPlace_u \lor \lnot \varPnPlace_v)$ to construct more
efficiently-solvable mutex groups.  Pairwise mutexes can yield a
quadratic number of constraints that are expensive to
solve~\cite{rintanen2012planning}.  Instead, we reduce the number of
constraints by constructing mutex groups of the form,
\begin{equation}
  \sum_{i\in \mathrm{group}} \varPnPlace_i \leq 1
  \label{eq:pn:placemutex}
  \; ,
\end{equation}
which are directly handled by SMT solvers with pseudo-Boolean
support~\cite{de2008z3} and MIP solvers.
Constructing maximal mutex groups is
PSPACE-Complete~\cite{fisler2018fact}, so we take a greedy approach
that progressively grows a mutex group when we can add another
variable that is mutually exclusive with all current group members.

From a mutex group, we can in some cases identify a more precise and
efficiently solvable one-hot invariant
when some place in the group is initially true and every transition
that disables a place in the group enables another place in the
group. The resulting one-hot invariant is,
\begin{equation}
  \sum_{i\in \mathrm{group}} \varPnPlace_i = 1
  \; .
  \label{eq:pn:placemutex}
\end{equation}

Experiments in \autoref{expe:inv} show that relaxation
\eqref{eq:pn:relaxed:dynamics} generates similar invariant counts as
baselines for tested domains while further supporting numeric domains.

\subsection{Infeasibility Explanations from Linear Relaxations} \label{ssec:infeasibility-expl}
We apply the PN reachability relaxation \eqref{eq:pn:relaxed:dynamics}
to identify and \emph{explain} planning infeasibilities.
Planning is infeasible when there is no plan from the start to the
goal.
We consider a planning infeasibility explanation to be a minimal set
of clauses causing the infeasibility, i.e., leading to unsatisfiable
constraints.

We check planning feasibility using relaxation
\eqref{eq:pn:relaxed:dynamics} under the goal $\plangoal$,
\begin{equation} \label{eq:goalinf}
  \plangoal \models
  \left(
    \discstep{\marking}{\stephorizon} \neq \discstep{\marking}{0} + \weighta\relaxedtransitions + \slackpv - \slacknv
  \right)
\; .
\end{equation}
By \autoref{prop:inf}, infeasibility in \eqref{eq:goalinf} means
the planning problem is infeasible.

Infeasibility explanations relate to unsatisfiable
cores~\cite{liffiton2008algorithms} and Irreducible Inconsistent
Subsystems (IISs)~\cite{chinneck2008feasibility}, which both identify
minimal, conflicting sets of constraints. 
Identifying all such inconsistencies is computationally hard
\cite{amaldi2003maximum}.

We focus on finding inconsistent sets of goal conditions. Goal
conditions offer potentially useful feedback since they directly
relate to an originally specified planning problem, while other
constraints arise from an involved sequence of transformations.
Specifically, an explanation $E \subseteq \plangoal$ is a subset of
goal conditions causing \eqref{eq:pn:relaxed:dynamics} to be
infeasible, but where \eqref{eq:pn:relaxed:dynamics} is feasible for
any strict subset of $E$.  We want to find all such subsets $E$.

The difficulty of finding all inconsistent goal sets depends on the
number of goal conditions.  For few goal conditions, we may enumerate
combinations in increasing cardinality, terminating a branch on
infeasibility, and collecting all infeasible combinations.  For many
goal conditions, equation \eqref{eq:pn:relaxed:dynamics} extends to an
integer program~\cite{amaldi2003maximum} that minimizes the number of
goal conditions to disable,
\begin{equation}
  \begin{split}
  \min \quad & \sum_{i=1}^{\ell} y_i \\
  \textrm{s.t.}
  \quad &
    \discstep{\marking}{\stephorizon} = \discstep{\marking}{0} +
          \weighta\relaxedtransitions + \slackpv - \slacknv \\
  & \lnot y_i \implies \plangoal_i,\quad \forall i \in 1,\ldots,\ell
    \; ,
  \end{split}
  \label{eq:inf:mip}
\end{equation}
\noindent where Boolean decision variables $y_i$ disable corresponding
goal conditions $G_i$.  Finding alternative solutions to
\eqref{eq:inf:mip} by asserting previously disabled conditions are not
jointly disabled yields sets of inconsistent goal conditions.

Experiments in \autoref{expe:inf} show that relaxation
\eqref{eq:pn:relaxed:dynamics} outperforms all baselines at
infeasibility detection for tested domains and further explains
infeasibilities.

\subsection{Forward and Backward Reachable Sets} \label{ssec:reachability-sets}

Next, we compute per-step reachable sets forward from the start and
backward from the goal.
These sets are optimistic approximations: states outside the set are
unreachable, while states in the set are possibly (but not
necessarily) reachable.
We represent a reachable set using bindings for a subset of state
variables, so the set contains states not conflicting with the
bindings.
Reachable sets help identify constants, eliminate decision variables,
and bound the minimum step horizon $\stephorizon$.

We describe reachable set propagation for forward reachability.
Backward reachability is similar but swaps use of preconditions with
effects and start with goal states.  We propagate known bindings from
one step to the next according to which transitions can or cannot
fire.
Conceptually, this approach is similar to planning graphs and other
unfolding techniques\cite{blum1997fast,hickmott2007planning}.
However, we propagate sets via \emph{partial evaluation} support
complex expressions and numeric state.

Set propagation uses functions \pval{} and \psat{}.
Function \pval{} takes an SMT-LIB~\cite{BarFT-RR-25} expression and
variable bindings to produce an equivalent (partially evaluated)
expression by substituting bindings and algebraically simplifying.
Function \psat{} is an incomplete procedure for Boolean and arithmetic
satisfiability that uses unit propagation and pure literal
elimination~\cite{marques1999grasp,moskewicz2001chaff} for
satisfiability tests that are
polynomial-time~\cite{karp1972reducibility} but may return
\DataSty{UNKNOWN}.

The first step of reachable set propagation identifies which
transitions cannot fire.  For transition $\varPnTrans_j$, if its
precondition $e$ under currently known bindings $\pbindingsstep{k}$ is
\DataSty{UNSAT}, the transition cannot fire, and we add
$\lnot\discstep{\varPnTrans_j}{k}$ to $\pbindingsstep{k}$,
\begin{equation}
  \left(
    \DataSty{UNSAT} = \psat\left(\pval
      \lmtuple{e}{\pbindingsstep{k}}
    \right)
  \right)
  \implies \lnot\discstep{\varPnTrans_j}{k}
  \; .
\end{equation}
The second step of propagation tests whether bound state variables can
change at the next step.
For bound place $\varPnPlace_i$, if its update $e$ from
\eqref{eq:pn:numflow} or \eqref{eq:pn:boolflow} and distinct next
value under bindings $\pbindingsstep{k}$ are \DataSty{UNSAT}, we
propagate the binding,
\begin{equation}
  \begin{split}
  & \left(
    \DataSty{UNSAT} =
    \psat\left(
      \pval
        \left(
          e \land
          \left( \discstep{\varPnPlace_i}{k} \neq \discstep{\varPnPlace_i}{k+1}\right)
          ,\:
          \pbindingsstep{k}
      \right)
    \right)
  \right) \\
  & \qquad \qquad \qquad \qquad \qquad \implies
  \left( \discstep{\varPnPlace_i}{k}=
          \discstep{\varPnPlace_i}{k+1}\right) \; .
\end{split}
\end{equation}

The result of propagating reachable sets to a fixpoint is a mapping of
per-step variable bindings forward from the start and backward from
the goal.
Any bindings present in the forward reachability fixpoint represent
constants.
We use these bindings to eliminate the decision variable for any known
binding in the constraint formulation.
A lower bound $\stephorizon$ on steps in feasible plans is the maximum
of the forward reachable step consistent with the start and the
backward reachable step consistent with the goal.

\subsection{Solutions via SMT and MILP Solvers} \label{sec:solve}

We solve constraints to find plans using SMT~\cite{de2008z3} or
MILP~\cite{gurobi} solvers.
Different solver algorithms or implementations perform better for
different problems, suggesting advantages for a portfolio of
solvers~\cite{rungta2022billion,combrink2025comparative}.
SMT solvers directly support the high-level constraint expressions
from the previous subsections~\cite{BarFT-RR-25}, and we leverage
incremental solving to progressively extend the step horizon.
MILP solvers operate on separately-specified linear constraints, so we
translate the SMT-LIB expressions to conjunctions of linear
constraints.
We first describe extending the step horizon and then translating
SMT-LIB to MILP constraints.

\subsubsection{Incrementally Solving Horizon Extension}
Generally, constraint-based planners check satisfiability for a
bounded number of steps $\stephorizon$ and increase $\stephorizon$
until finding a valid plan.
Early formulations considered each bound $\stephorizon$
separately~\cite{kautz1992planning}.
Madagascar identified commonalities between different
bounds $\stephorizon$~\cite{rintanen2014madagascar}.
Prior use of SMT solvers for planning used the assertion
stack~\cite{dantam2018tmp,dantam2018tmkit}.
However, popping an assertion scope may cause solvers to discard prior
work~\cite{bjorner2018programming}.
We propose instead to use per-check assumptions for more effective
incremental solving.

Horizon extension involves (1) declaring new decision variables and
constants\footnote{The SMT-LIB standard uses the term
  ``constant''~\cite{BarFT-RR-25} for what optimization literature
  calls a ``decision variable.''}, (2) asserting the transition
function at that step, and (3) checking satisfiability of the goal at
the next step.
Extending the horizon to $k$, we declare decisions variables for each
place $\discstep{\varPnPlace_i}{k}$ and transition
$\discstep{\varPnPlace_i}{k}$ at that step.  If forward reachability
(see \autoref{ssec:reachability-sets}) indicated a known binding, then
we declare instead constants (i.e., \DataSty{true}, \DataSty{false},
or a number).
Next, we assert that the transition function (see
\autoref{ssec:pn-constr}) holds from steps $k-1$ to $k$ and that
invariants (see \autoref{ssec:inv}) hold at step $k$.
Finally, we check satisfiability with assumptions
(\KwSty{check-sat-assuming}) that the goals holds at step $k$ and
known bindings from backwards reachability (see
\autoref{ssec:reachability-sets}) hold at steps prior to $k$.  We
repeatedly increase the horizon until goal is satisfied.

\subsubsection{Transforming SMT-LIB to MILP Constraints}

We translate SMT-LIB to MILP constraints using a variation of the
Tseitin~\cite{tseitin1983complexity} or Plaisted-Greenbaum
transformation~\cite{plaisted1986structure} transformation to
construct conjunctive normal form (CNF) expressions, and we further
eliminates some decision variables and apply indicator constraints for
disjunctions of linear relations~\cite{belotti2016handling}.

The Tseitin and Plaisted-Greenbaum transformations construct CNF by
recursively introducing auxiliary variables for subexpressions.
Tseitin introduces auxiliary variables $\varTsei$ that equal
subexpressions
$\operatorname{op}{\lmtuple{\varLiteral_1}{\ldots}{\varLiteral_n}}$
over literals $\varLiteral_1, \ldots,\varLiteral_n$.
Plaisted and Greenbaum observed that satisfiability requires only one
direction of the equality biconditional, which from negation normal
form, becomes,
\begin{equation}
  \operatorname{op}{\lmtuple{\varLiteral_1}{\ldots}{\varLiteral_n}}
  \quad \leadsto \quad
  \varTsei \implies
  \operatorname{op}{\lmtuple{\varLiteral_1}{\ldots}{\varLiteral_n}}
  \; .
  \label{eq:pg}
\end{equation}
Equation \eqref{eq:pg} reduces directly to conjunctions of
disjunctions.

We further eliminate auxiliary variables in some cases.
Conjunctions involving auxiliary variables $\varTsei_i$ require all to
hold, so we eliminate the $\varTsei_i$'s by rewriting their
implications using $\varTsei'$,
\begin{equation}
  \left(\bigwedge \varTsei_i\right)
  \land \left(\bigwedge \varLiteral_j\right)
  \quad \leadsto \quad
  \left(\varTsei' \implies
    \bigwedge \varLiteral_j
  \right),
  \quad
  \left(
    \varTsei_i \implies \varExp
    \ \leadsto \
    \varTsei' \implies \varExp
  \right)
  \; .
\end{equation}
Disjunctions involving a single auxiliary variable $\varTsei$ require
either $\varTsei$ or some literal $\varLiteral_j$ to hold, so we
eliminate $\varTsei$ by replacing implications of $\varTsei$
with $\varTsei'$,
\begin{equation}
  \varTsei \lor \bigvee \varLiteral_j
  \quad \leadsto \quad
  \left(
    \varTsei \implies \varExp
    \ \leadsto
    \varTsei' \implies
    \varExp  \lor \bigvee \varLiteral_j
  \right)
  \; .
\end{equation}
Reducing the resulting CNF to linear constraints is
well-established~\cite{karp1972reducibility}.

Support for logical expressions over numeric relations
($\VEC{a}^T\VEC{x} \leq b$) requires additional consideration.
Generally, such numeric relations pose disjunctive
programs~\cite{balas1979disjuntive}.  We incorporate numeric relations
into the CNF transformation by creating auxiliary variables that imply
the relation,
\begin{equation}
\varTsei \implies \VEC{a}^T\VEC{x} \leq b
\; .
\label{eq:cnf:relation}
\end{equation}
Some solvers support implication \eqref{eq:cnf:relation} directly as
an indicator constraint~\cite{gurobi}.
For solvers lacking indicator
constraints~\cite{huangfu2018parallelizing}, we apply the Big M
method~\cite{bazaraa2009linear} coupled with interval
arithmetic~\cite{jaulin2001applied} to find tight $M$ values,
\begin{equation}
  \VEC{a}^T\VEC{x} + M \varTsei  \leq b + M\; ,
  \quad\quad
  M = s(\overline{\VEC{a}^T\VEC{x}} - b)
  \; ,
  \label{eq:cnf:bigm}
\end{equation}
\noindent where $\overline{\VEC{a}^T\VEC{x}}$ is the upper bound of
$\VEC{a}^T\VEC{x}$ computed via interval arithmetic based on constant
coefficients $\VEC{a}$ and box constraints for decision variables
$\VEC{x}$, and $s > 1$ is a scaling factor to tolerate floating point
error (our implementation uses $s=2$).

\subsection{Updates and Incremental Sequential Planning}
\label{sec:incremental}

We consider two forms of updates for sequential planning problems:
changing goals and adding constraints.
Changing goals may be useful when an initial problem is infeasible,
yet a subset of conflicting goals (see
\autoref{ssec:infeasibility-expl}) is still desirable.
Adding constraints may be useful when a candidate plan is undesirable
due to an initially unspecified requirement.
While more general updates are possible---e.g., changes to state
variables, preconditions, or effects---these restricted update forms
are efficient  to implement via incremental constraint solving.

\paragraph{Changing Goals}
A change of goals operates on update $\updates$ containing a list of added
clauses $\plangoalclause_i^+$ and deleted clauses
$\plangoalclause_j^-$,
\begin{equation}
  \updates =
  \lmtuple{\plangoalclause_1^+}{\ldots}{\plangoalclause_m^+},\quad
  \lmtuple{\plangoalclause_1^-}{\ldots}{\plangoalclause_n^-}
  \; .
\end{equation}
Goal changes modify the PN's goal marking, but the PN is otherwise
unchanged.
The relaxed marking equation (\autoref{ssec:relaxed}) and
invariants (\autoref{ssec:inv}) are unmodified.
We recheck feasibility (\autoref{ssec:infeasibility-expl}).
Forward reachability is unmodified, and we recompute backward
reachability (\autoref{ssec:reachability-sets}).
Finally, the updated goal and backward reachability bindings become the
per-check assumptions (\autoref{sec:solve}).

\paragraph{Adding Constraints}
An addition of constraints operates on update $\updates{}$ containing
a list of logical expressions or linear relations $e_i$ that must hold
across all steps,
\begin{equation}
  \updates = \lmtuple{e_1}{\ldots}{e_n}
  \; .
\end{equation}
This update creates additional constraints for the marking equation
(\autoref{ssec:pn-constr}),
  $\discstep{e_1}{k+1} \land \ldots \land \discstep{e_n}{k+1}$,
and correspondingly for the linear relaxation (\autoref{ssec:relaxed}),
  $\discstep{e_1}{h} \land \ldots \land \discstep{e_n}{h}$.
We must recompute the invariants (\autoref{ssec:inv}), goal
reachability (\autoref{ssec:infeasibility-expl}), and reachable sets
(\autoref{ssec:reachability-sets}) under these new constraints, noting
that these preprocessing steps are all polynomial time.
Each additional constraint $e_i$ along with new invariants and known
bindings from forward reachability become new assertions for the
constraint solver, noting that these new assertions are typically
efficiently
handled~\cite{bjorner2018programming,de2008z3,een2003extensible}.
Finally, any new known bindings from backward reachability became new
assertions for satisfiability checks.

Experiments in \autoref{expe:int} show that the incremental constraint
approach outperforms all baselines for sequential planning in the
tested domains.

\section{Experiments and Discussion} \label{sec:result}

We evaluate our approach for (1) invariant generation, (2)
infeasibility detection, and (3) sequential planning, and we compare
to the following efficient and state-of-the-art baselines:
constraint-based Madagascar (Mp)~\cite{rintanen2014madagascar},
heuristic Fast Downward (FD) \cite{helmert2006fast}, numeric Metric-FF
(FF) \cite{hoffmann2003metric} and ENHSP \cite{scala2016interval}.
Our results show that our approach found similar numbers of invariants
as Mp and FD in classical domains and outperformed baselines on
infeasibility detection and sequential planning for the tested
domains.

Our implementation extends TMKit~\cite{dantam2018tmkit}. We solve
constraints for the relaxation \eqref{eq:pn:relaxed:dynamics} and
planning (\autoref{sec:solve}) using Z3 \cite{de2008z3} and Gurobi
\cite{gurobi}.
The tested domains come from past International Planning Competitions
(IPCs)~\cite{ipc2016,ipc2023}, are auto-generated by
\cite{seipp-et-al-zenodo2022}, or randomly generated by ourselves.
We ran our benchmarks using one core of an AMD Ryzen 9 7950X3D CPU
under Debian 12.

\subsection{Invariant Generation} \label{expe:inv}

\begin{wrapfigure}[17]{r}{.4\textwidth}
  \vspace{-28pt}
  \includegraphics[width=\linewidth]{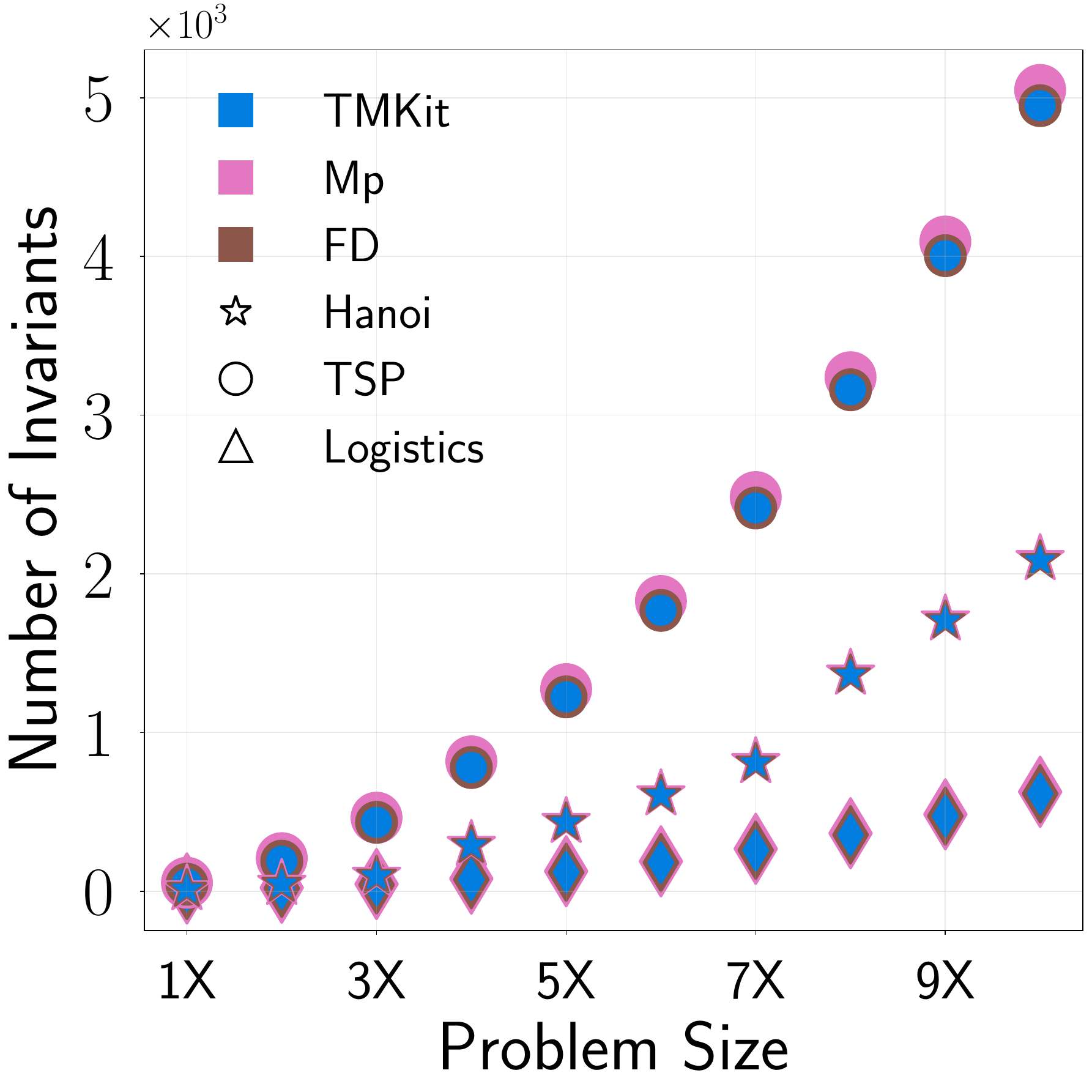}
  \caption{Two-literal invariants from our PN relaxation and
    baselines.  All methods produce similar numbers of
    invariants.
  }
  \label{img:invariants}
\end{wrapfigure}

We evaluate invariant generation from \autoref{ssec:inv} for
increasing problem size in three combinatorial problems.
We compare against Mp and FD because the other baselines do not
produce invariants. We only evaluate classical (not numeric) domains,
because Mp and FD only support classical domains.

\autoref{img:invariants} plots invariant counts.
Our PN relaxation produces similar numbers of invariants as Mp and FD
in the tested domains, and \autoref{prop:inf} guarantees invariant
validity.
Mp yields slightly more invariants than ours and FD, while
FD produces the same number of invariants as ours.

Each evaluated planner generates invariants differently, and the
additional invariants from Mp can be explained by its unique invariant
synthesis.
Our method checks inconsistency of relaxed PN flow to identify
pairwise invariants.
FD checks the monotonicity of each candidate across actions and
iteratively strengthens a non-monotonic invariant by adding axioms to
maintain its monotonicity~\cite{helmert2009concise}.
Mp checks a candidate invariant through operator regressions while
maintaining the consistency of the regressions with other candidate
invariants~\cite{rintanen2008regression}, and a similar iterative
approach~\cite{rintanen2008regression} generates more classes of
invariants than FD~\cite{helmert2009concise}.
Our method is similar to FD in that we check candidate invariants
independently.

\subsection{Infeasibility Detection} \label{expe:inf}

\begin{figure} [t]
    \centering
    \includegraphics[width=1\linewidth]{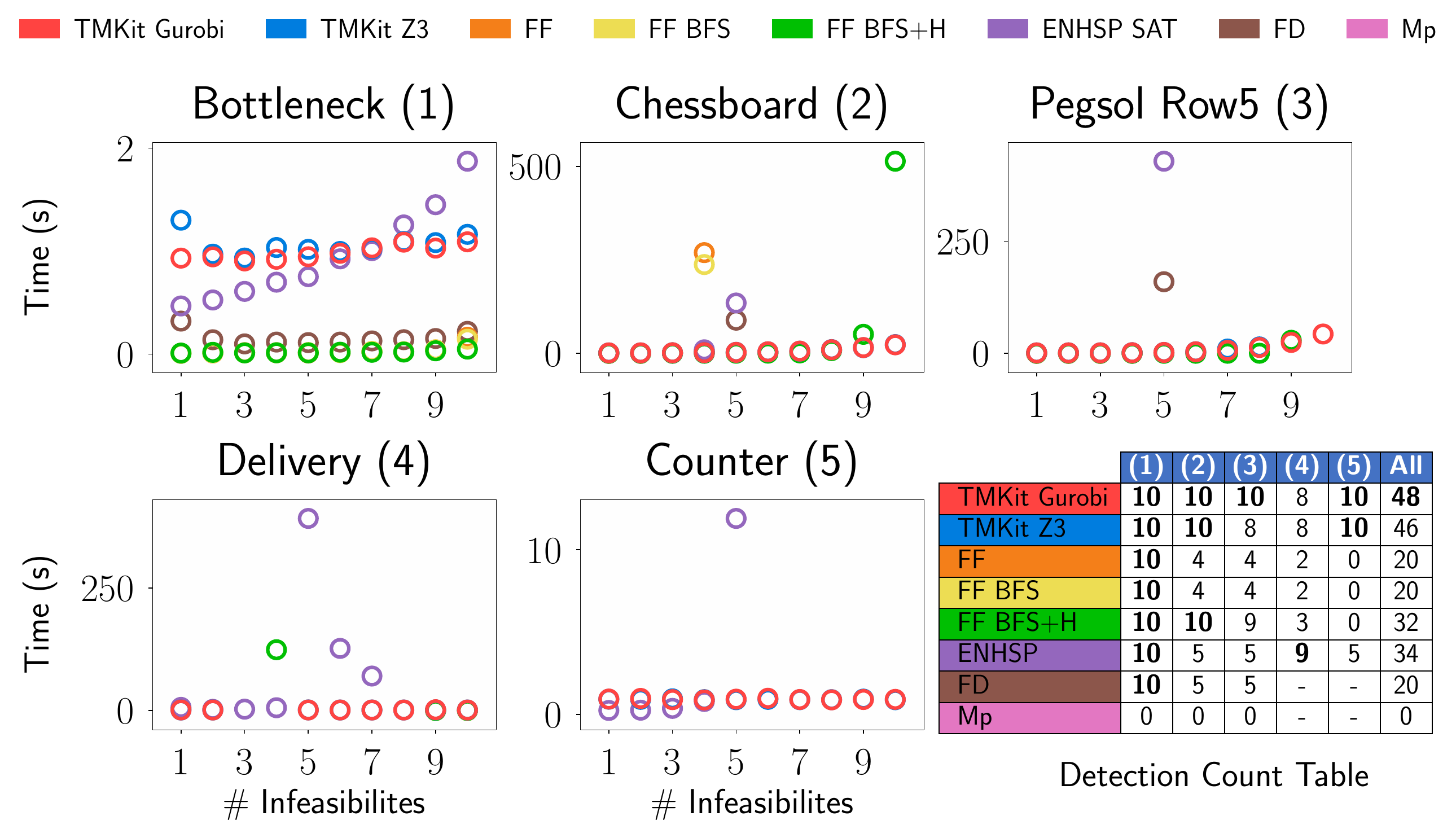}
    \vspace{-15pt}
    \caption{Average infeasibility detection times over $10$ iterations and   counts. Absent data points exceeded a $10$-minute timeout or
      \qty{8}{\gibi\byte} memory limit. Dashes in the table indicates
      unsupported numeric domains for Mp and FD. }
    \label{img:inf-times}
    \label{table:infeasibility-detection}
    \vspace{-9pt}
\end{figure}

We evaluate infeasibility detection from
\autoref{ssec:infeasibility-expl} on classical problems from
Unsolvability IPC 2016~\cite{ipc2016}, and
we auto-generate additional unsolvable numeric problems (Delivery and
Counter). Each tested domain has $10$ problems of increasing size.
Planners have a 10 minute timeout \qty{8}{\gibi\byte} memory limit.
Our reachability relaxation identifies up to twice as many
infeasibilities as the baselines in the tested domains, finding $48$
infeasibilities out of $50$ cases (see
\autoref{table:infeasibility-detection}).

Our relaxation is particularly robust at infeasibility detection in
numeric domains. FF's performance is comparable to ours in
combinatorial problems but worse on numeric problems. ENHSP exhibits
an average performance on both tracks, although it finds one more
infeasibility than ours in the Delivery domain.

Different procedures contribute to the varying detection capabilities.
Heuristic planners use a heuristic function (relaxation) to detect
dead-ends that cannot reach the goal, though missed dead-ends may
result in an expensive, exhaustive search.
On the other hand, our LP relaxation enables polynomial time
feasibility tests, and \autoref{img:inf-times} indicates that LP
feasibility closely approximates planning feasibility. From
\autoref{prop:inf}, feasible domains never produce infeasible LPs.

Beyond infeasibility detection, our approach also explains
infeasibility in direct connection to original planning domains.
Our explanations are minimal sets of mutually unreachable goals (see
\autoref{ssec:infeasibility-expl}), which convey helpful information
to resolve conflicts in subsequent plan updates. Heuristic approaches
report different infeasibility information in the form of
unsolvability certificates~\cite{eriksson2017unsolvability},
i.e. dead-ends, which may offer some details on how to resolve
infeasibilities.

\subsection{Sequential Planning} \label{expe:int}

\begin{figure} [t]
    \centering
    \includegraphics[width=1\linewidth]{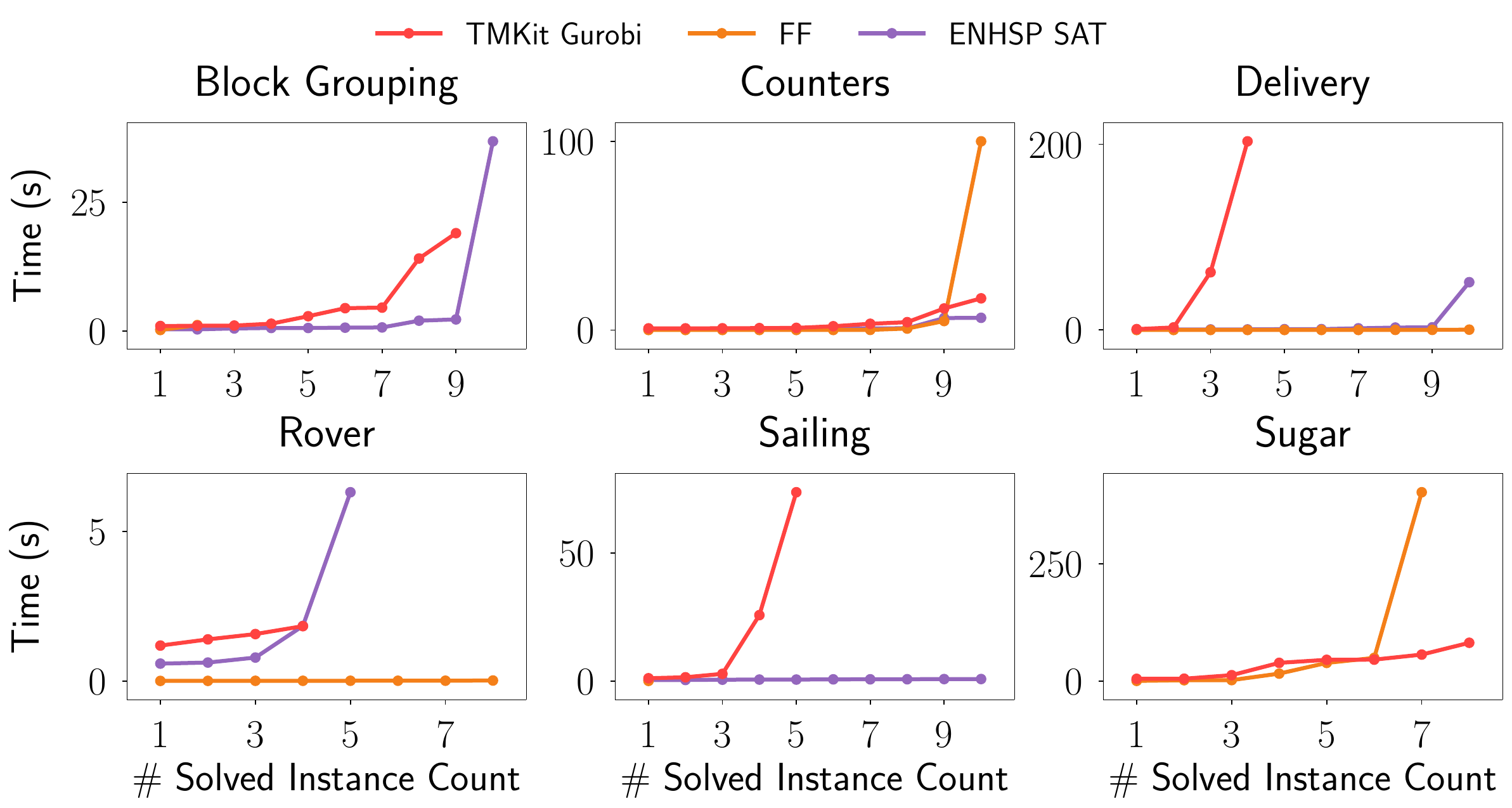}
    \vspace{-15pt}
    \caption{Average one-shot planning times over $10$ iterations. The figures show the numbers of
      instances solved over time in numeric domains. Excluded data
      points exceeded the $10$-minute timeout or \qty{8}{\gibi\byte}
      memory limit.}
    \label{img:runtimes}
    \vspace{-12pt}
\end{figure}

\begin{figure} [t]
    \centering
    \includegraphics[width=1\linewidth]{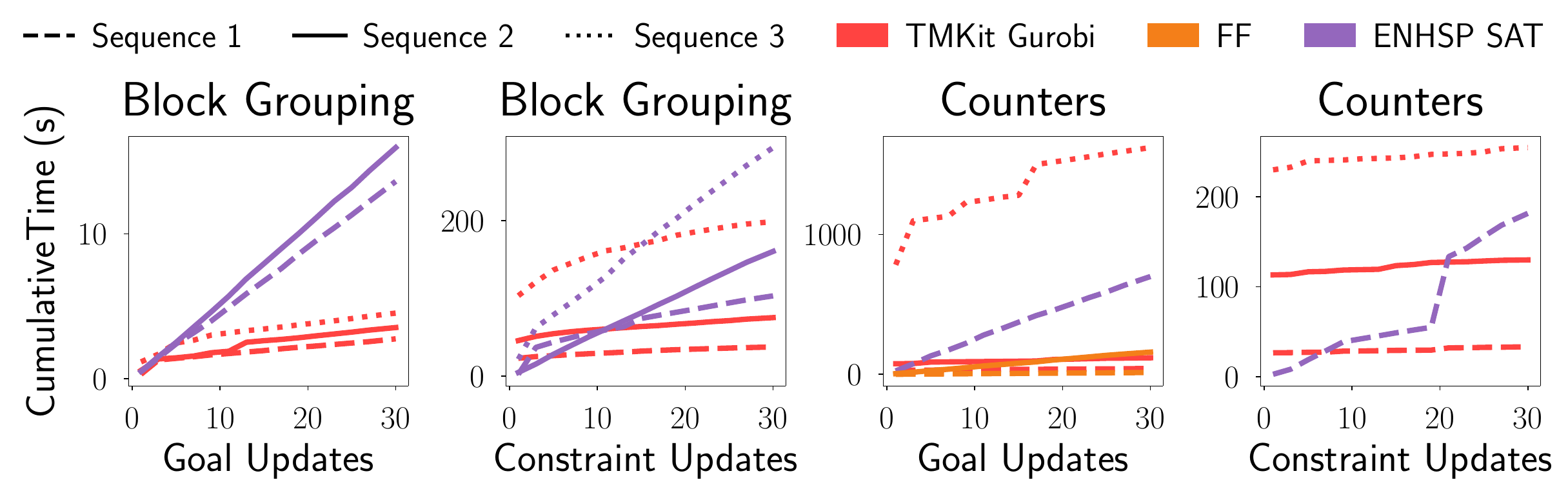}
    \vspace{-15pt}
    \caption{Average cumulative planning times over $10$ iterations 
      for sequential planning. TMKit
      Gurobi outperforms the baselines in most sequences.  Excluded
      data points exceeded the $30$-minute timeout or
      \qty{8}{\gibi\byte} memory limit.  }
    \label{img:updates}
    \vspace{-12pt}
\end{figure}

We evaluate sequential planning from \autoref{sec:incremental} by
first considering time for traditional one-shot planning followed by
cumulative time for sequential updates and planning.
For one-shot planning, we evaluate six Simple Numeric
Planning (SNP) domains from IPC 2023~\cite{ipc2023}. Each domain
consists of $10$ difficulty-varying problems.
For sequential planning, we benchmark the planners on randomly
generated update sequences in the Block Grouping and Counters
domains. The sequences increase in the problem size. We consider two
forms of plan updates: goal changes and constraint additions. Each
sequence consists of $30$ updates.
We show results for TMKit Gurobi, FF, and ENHSP SAT, which were most
efficient variations of their respective planners on these tests.

Relative one-shot planner performance is domain dependent (see
\autoref{img:runtimes}).
Our approach outperforms
FF in the Block Grouping, Counters, Sailing, and Sugar domains and
outperforms ENHSP SAT in the Sugar domain.
Invariant information and efficient numeric constraint solving lead to
our system's one-shot performance.
Though one-shot performance was not the main focus of our work, we do
see efficient solutions for some tested problems.

We outperforms the baselines in sequential planning in the tested
domains (see \autoref{img:updates}).
TMKit Gurobi takes more time for some initial problems, but cumulative
times over sequential updates are less than the baselines.
In addition, while our constraint-based framework supports additional
constraints, some baselines including Mp, FD, and FF could not handle
these updates due to lack of support for PDDL global constraints.

Our outperformance for sequential planning comes from incremental
solving and warm starts. The baselines only support one-shot planning
and thus replan per update.
In contrast, we retain the constraints and reuse the previous search
information to guide the next search.  Gurobi as the underlying
constraint solver applies the previous solution as a warm start for
the next search, resulting in efficient updates when changing goals or
adding constraints~\cite{gurobi}.

\subsection{Limitations}

Our PN relaxation depends on checking LP infeasibility, and floating
point error can lead to spurious
infeasibilities~\cite{gurobi,FicoXpressOptimizerManual}.
Rational arithmetic (typical of SMT solvers~\cite{de2008z3})
eliminates floating point error at additional computational cost,
though does not address domains requiring irrational
coefficients---e.g., $\pi$, trigonometric functions.

The incremental constraint-based approach may present limitations in
certain cases.
Constraint-based planning faces performance challenges under large
numbers of grounded state variables or long
horizons~\cite{elahi2024optimizing}, and which our current work does
not address.
Conversely, on easy problems, constraint solving may impose
uncompensated overhead; \autoref{img:updates} shows FF outperforms our
system on the smallest Counters problem.
Incremental solving and warm starts help only after solving the first
problem in a sequence, and heuristic approaches are often
highly-effective at one-shot planning.
While our approach performed effectively in the experimental tests,
there are cases it does not address.

\section{Conclusion}

We presented an approach for infeasibility explanation and sequential
planning that combines a relaxation of Petri Net (PN) reachability and
incremental constraint solving.
The PN reachability relaxation is a linear program (LP); feasibility
of this LP is necessary for planning feasibility, offering polynomial
time invariant and plan infeasibility tests.
Further, the LP offers infeasibility explanations in terms of
conflicting constraints.
The incremental constraint approach reuses prior effort when solving
updated problems for sequential planning.
Empirically, the PN relaxation produced similar numbers of invariants
and identified up to $2\times$ more infeasible cases with lower
detection times and better scaling compared to the baseline approaches
on the tested problems.  Finally, our incremental approach for plan
updates showed improved cumulative running times compared to the
baselines in sequential planning on the tested problems.

\clearpage

%
\bibliographystyle{splncs04}
\bibliography{wafr,dyalab-pubs, dyalab-refs}

\end{document}